\documentclass[letterpaper, 10 pt, conference]{ieeeconf}
\IEEEoverridecommandlockouts

\usepackage[margin=0.75in]{geometry}

\usepackage[compress]{cite}
\makeatletter
\let\NAT@parse\undefined
\makeatother
\usepackage[colorlinks=true]{hyperref}

\usepackage[table]{xcolor} 
\usepackage[section]{placeins} 

\usepackage{pifont}  
\usepackage{graphics} 
\usepackage{amsmath} 
\usepackage{mathtools} 
\usepackage{amssymb}  
\usepackage{makeidx} 
\usepackage{graphicx} 
\graphicspath{{figures/}}
\usepackage{multicol} 
\usepackage[bottom]{footmisc} 
\usepackage[utf8]{inputenc} 
\usepackage{multirow}
\usepackage{booktabs} 
\usepackage{adjustbox}
\usepackage{caption}
\usepackage{makecell} 
\usepackage{tikz}
\usetikzlibrary{positioning}
\usetikzlibrary{shapes,snakes}
\usepackage{bm} 
\usepackage{float} 
\usepackage{color} 
\usepackage{empheq} 
\usepackage{threeparttable} 
\usepackage{graphicx}

\usepackage{tikz}

\definecolor{mypink2}{RGB}{139,0,139}
\definecolor{mygreen}{RGB}{0,110,51}

\usepackage[a-1b]{pdfx}

\usepackage{xurl} 
\usepackage{todonotes}

\usepackage{paralist}
\newtheorem{thm}{Theorem}

\newtheorem{lem}[thm]{Lemma}

\newcommand\scalemath[2]{\scalebox{#1}{\mbox{\ensuremath{\displaystyle #2}}}}

\usepackage{mathtools}
\DeclarePairedDelimiterX{\norm}[1]{\lVert}{\rVert}{#1}

\newcommand{\eg}{\textit{e}.\textit{g}.}

\newcommand{\ie}{\textit{i}.\textit{e}.}

\DeclareGraphicsExtensions{.eps,.pdf,.png,.jpg,.JPG,.gif}

\begin{document}
%

\title{\LARGE \bf Consistent and Efficient MSCKF-based LiDAR-Inertial Odometry with Inferred Cluster-to-Plane Constraints for UAVs}
%
\author{Jinwen Zhu$^{\dag*}$, 
        Xudong Zhao$^{\dag*}$,
        Fangcheng Zhu$^\dag$,
        Jun Hu$^\dag$,
        Shi Jin$^\dag$, 
        Yinian Mao$^\dag$, 
        Guoquan Huang$^\ddag$
\thanks{$^{*}$Equal contribution.}
\thanks{$\dag$Meituan UAV, Beijing, China (e-mail: \{zhaoxudong09, zhujinwen, zhufangcheng, hujun11, jinshi02, maoyinian\}@meituan.com).}
 \thanks{ $\ddag$Dept. of Mechanical Engineering, Computer and Information Sciences, University of Delaware, Newark, DE (email: ghuang@udel.edu).}
}
\maketitle

\begin{abstract}
Robust and accurate navigation is critical for Unmanned Aerial Vehicles (UAVs) especially for those with stringent Size, Weight, and Power (SWaP) constraints. However, most state-of-the-art (SOTA) LiDAR-Inertial Odometry (LIO) systems still suffer from estimation inconsistency and computational bottlenecks when deployed on such platforms. To address these issues, this paper proposes a consistent and efficient tightly-coupled LIO framework tailored for UAVs. Within the efficient Multi-State Constraint Kalman Filter (MSCKF) framework, we build coplanar constraints inferred from planar features observed across a sliding window. By applying null-space projection to sliding-window coplanar constraints, we eliminate the direct dependency on feature parameters in the state vector, thereby mitigating overconfidence and improving consistency. More importantly, to further boost the efficiency, we introduce a parallel voxel-based data association and a novel compact cluster-to-plane measurement model. This compact measurement model losslessly reduces observation dimensionality and significantly accelerating the update process. Extensive evaluations demonstrate that our method outperforms most state-of-the-art (SOTA) approaches by providing a superior balance of consistency and efficiency. It exhibits improved robustness in degenerate scenarios, achieves the lowest memory usage via its map-free nature, and runs in real-time on resource-constrained embedded platforms (\eg, NVIDIA Jetson TX2).
\end{abstract}

\section{Introduction}
\label{sec:intro}
Robust, accurate, and efficient state estimation is prerequisite for Unmanned Aerial Vehicles (UAVs), particularly in safety-critical applications such as commercial drone delivery~\cite{faaPackageDelivery}. Valued for its immunity to lighting variations and precise ranging, LiDAR has become a primary sensor for autonomous aerial navigation. While LiDAR provides precise ranging, deploying LiDAR-Inertial Odometry (LIO) on UAVs is challenged by stringent Size, Weight, and Power (SWaP) constraints.
Currently, the Error-State Iterated Kalman Filter (ESIKF) framework (\eg, FAST-LIO2 \cite{xu2022fast}, Super-LIO \cite{wang2026super}) dominates due to its impressive efficiency. However, these methods fundamentally rely on a global map treated as a fixed prior. This architecture theoretically presents a critical  limitation: \textbf{inconsistency}. 
By treating the estimated map as a fixed prior with independent noise, these approaches neglect the cross-correlations between the current state and the map. This leads to ``over-confidence'', where the estimated uncertainty is significantly lower than the actual error~\cite{yuan2022efficient}. 

In contrast, the Multi-State Constraint Kalman Filter (MSCKF) framework~\cite{mourikis2007multi}---along with consistent treatments such as First-Estimate Jacobian (FEJ)~\cite{li2013high,geneva2020openvins}---offers a consistent alternative by exploiting relative constraints within a sliding window.
However, standard MSCKF implementations such as LINS~\cite{qin2020lins} still suffer from computational bottlenecks due to the state vector growing with the window size and the processing of massive LiDAR measurements.

To address these challenges, we propose a consistent and efficient LIO system tailored for resource-constrained platforms, \eg, UAVs. Our contributions are threefold:
\begin{itemize}
    \item To ensure consistency, we propose a tightly-coupled map-free MSCKF-based LIO. By marginalizing features via null-space projection, we avoid the over-confidence of ESIKF methods, yielding \textit{accurate and consistent} estimates.
    
    \item Aiming for computational efficiency, we design a parallel, adaptive voxel-based data association coupled with a novel, compact \textit{cluster-to-plane} measurement model. This formulation losslessly reduces the observation dimension by more than 85\%, effectively resolving the efficiency bottleneck of traditional LIO .
    
    \item We perform extensive validation in simulations and real-world flights, showing that our method achieves superior consistency compared to SOTA approaches and {\em real-time capability with extremely efficient memory usage}, even running on severely resource-constrained computers (\eg, NVIDIA Jetson TX2).
\end{itemize}

The remainder of this paper is organized as follows: 
After reviewing the related work, we present the system formulation in Section~\ref{sec:problem}. Section~\ref{sec:upd} details the proposed parallelized voxelization and efficient MSCKF update with inferred cluster-to-plane measurements. Extensive validation through simulations and real-world experiments is provided in Sections~\ref{sec:sim} and Section~\ref{sec:exp}. Finally, Section~\ref{sec:concl} concludes the proposed framework and discusses the future work.

\section{Related Work}\label{sec:related_works}

The field of LiDAR-Inertial Odometry (LIO) and SLAM has witnessed a proliferation of high-performance systems in recent years. These can be broadly categorized into optimization-based and filtering-based approaches.

The \textit{optimization-based paradigm}, which relies on iterative error minimization, includes seminal works like LOAM~\cite{zhang2014loam}. LOAM pioneered real-time odometry through efficient plane-edge feature matching. Building on this, LIO-SAM~\cite{shan2020lio} employs a factor graph to tightly fuse IMU pre-integration with keyframe scan-matching for enhanced robustness. More recently, SR-LIO~\cite{yuan2024sr} attempts to better handle motion distortion via sweep reconstruction. Other lightweight variations like CT-ICP~\cite{dellenbach2021cticp} extend these concepts to continuous-time trajectory estimation. FF-LINS \cite{ff-lins} proposed a frame-to-frame LiDAR-inertial state estimator under factor graph optimization framework meanwhile maintaining consistency.

Conversely, \textit{filtering-based methods} sequentially update the state estimate within a recursive framework, often prioritizing computational efficiency. LINS~\cite{qin2020lins} integrates IMU and LiDAR data using an Iterated Extended Kalman Filter (IEKF). The field was significantly advanced by the FAST-LIO family\cite{xu2021fast,xu2022fast,bai2022faster,zhu2024swarm}: FAST-LIO~\cite{xu2021fast} reformulated the Kalman gain to depend on the measurement dimension, and FAST-LIO2~\cite{xu2022fast} introduced the ikd-tree for efficient map management. To further reduce computational load, Faster-LIO~\cite{bai2022faster} employs parallel sparse incremental voxels, and iG-LIO~\cite{chen2024ig} leverages an incremental GICP-based frontend to accelerate the update process. VoxelMap~\cite{yuan2022efficient} further improves accuracy by using adaptive voxels to incrementally update local map planes with uncertainty modeling.

However, a critical limitation of many high-performance ESIKF systems (including ~\cite{xu2022fast, faster-lio2022ral, yuan2022efficient, chen2024ig}) is their failure to maintain a theoretically correct uncertainty measure. These methods typically treat the global map as a pre-built, fixed reference with independent noise. This assumption neglects the cross-correlations between the current state and the map, leading to ``over-confidence", where the estimated covariance fails to reflect the true error growth.

A recent line of work has focused explicitly on resolving this inconsistency. Methods like MINS~\cite{lee2023mins} and Puma-LIO~\cite{jiang2022lidar} achieve consistency by using null-space projection to reformulate point-to-plane residuals as relative constraints between frames, effectively removing the dependency on a global map.
The MSCKF framework, renowned for its efficiency and consistency in visual-inertial odometry, has also been adapted for LIO. LIC-FUSION~\cite{zuo2019lic} utilizes this framework but relies on single-frame feature extraction, leading to fragile result in sparse environments. Although recent works like MSC-LIO~\cite{zhang2024msc} and BA-LINS~\cite{tang2025ba} also employ MSCKF for consistent estimation, their data association typically involves expensive neighbor searches. This process leads to a computational bottleneck in large-scale environments, hindering their deployment on SWaP-constrained platforms. Our work addresses this specific gap by introducing an adaptive voxel-aided LIO with cluster-to-plane measurement model under MSCKF framework that ensures both consistency and extreme efficiency.

\section{System Overview} \label{sec:problem}

In this section, we briefly describe the problem formulation  within a sliding-window MSCKF framework, 
and provide an overview of the proposed consistent and efficient LIO system as shown in Fig.~\ref{fig:overview}.

\begin{figure}[t]
\centering
\includegraphics[width=0.48\textwidth]{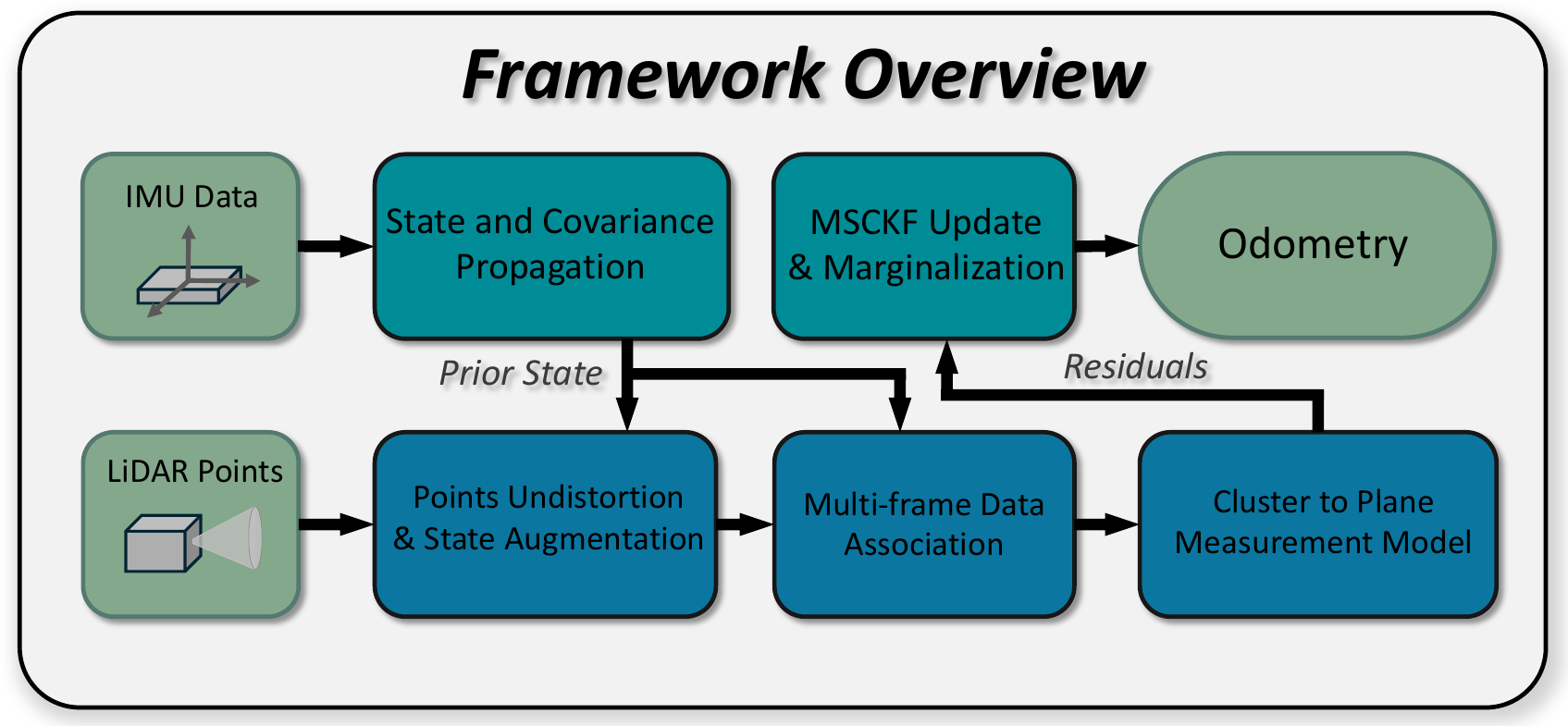} 
\caption{System diagram of the proposed MSCKF-based LIO.} 
\label{fig:overview}
    \vspace{-5mm}
\end{figure}

Within the efficient MSCKF framework~\cite{mourikis2007multi,geneva2020openvins}, IMU is used to propagate the state estimates and covariance, and LiDAR measurements are used for update.
The state vector of MSCKF framework includes $n$ cloned IMU poses $\mathbf x_{C_i}$ ($i=1,\cdots, n$) and the current IMU navigation states $\mathbf x_I$:

\begin{small}
\begin{align} \label{eq:state}
\mathbf x_k^T 
=& 
\begin{bmatrix}
\mathbf{x}_{C_1}^T & \cdots & \mathbf{x}_{C_n}^T & \mathbf x_{I_k}^T
\end{bmatrix}\\
=&
\begin{bmatrix}
\mathbf{x}_{C_1}^T & \cdots & \mathbf{x}_{C_n}^T &  
   ^G\mathbf{p}_{I_k}^T & ^G\mathbf{q}_{I_k}^T & ^G\mathbf{v}_{I_k}^T & \mathbf{b}_{g_k}^T & \mathbf{b}_{a_k}^T
\end{bmatrix} \notag
\end{align}
\end{small}%
\noindent where $\mathbf x_{C_i}^T=[^G\mathbf{q}_{C_i}^T,^G\mathbf{p}_{C_i}^T]$, represents the rotation and position corresponding to the $i$-th cloned IMU pose,
$^G\mathbf{p}_I, ^G\mathbf{q}_I, ^G\mathbf{v}_I$ are the position, rotation (unit quaternion) and velocity of current IMU state expressed in the global frame,
and $\mathbf{b}_g$ and $\mathbf{b}_a$ are the gyroscope  and accelerometer biases, respectively.


When a new IMU reading is available, we use the acceleration and angular velocity measurements to propagate the IMU navigation state based on the inertial kinematic model~\cite{trawny2005indirect}: 

\begin{small}
\begin{equation}\label{equ:state predict}
    \begin{aligned}
    {\mathbf{x}}_{k} = \mathbf f (\mathbf{x}_{k-1},\mathbf{a}_{m,k-1}, \boldsymbol{\omega}_{m,k-1}, \mathbf{n}_{k-1})
    \end{aligned}
\end{equation} 
\end{small}%
\noindent where $\mathbf{a}_m, \boldsymbol{\omega}_m$ is the measurement of acceleration and angular velocity, $\mathbf{n}$ is noise of the IMU measurement, which includes white Gaussian noise and random walk bias noise.

The state estimate propagates from  $k-1$ to $k$ is as follows:

\begin{small}
\begin{equation}\label{equ:state transition}
    \begin{aligned}
    {\hat{\mathbf{x}}}_{k|k-1} = \mathbf f (\hat{\mathbf{x}}_{k-1|k-1},\mathbf{a}_{m,k-1}, \boldsymbol{\omega}_{m,k-1}, \mathbf{0})
    \end{aligned}
\end{equation}
\end{small}

On the other hand, the covariance $\mathbf{P}_{k|k-1}$ is propagated based on the linearized IMU kinematic model of \eqref{equ:state predict}:

\begin{small}
\begin{equation}\label{equ:state covariance propagation}
    \mathbf{P}_{k|k-1}=\mathbf{\Phi}_{k-1}\mathbf{P}_{k-1|k-1}\mathbf{\Phi}_{k-1}^{\top}+\mathbf{Q}_{k-1}
\end{equation}
\end{small}%
\noindent where $\mathbf{\Phi}_{k-1}$ is the system Jacobian and $\mathbf{Q}_{k-1}$ is the discrete-time noise covariance.

When a new LiDAR pointcloud measurement becomes available, we utilize the IMU propagation to obtain the prior state estimate. To perform the MSCKF-based sliding window update, we clone the IMU pose $\mathbf x_{C_i}$ [see~\eqref{eq:state}] corresponding to the current LiDAR time  and augment the state and covariance (\ie, stochastic cloning).

Given the substantial volume of LiDAR data, directly processing raw points is computationally prohibitive. Therefore, we pass the accumulated point clouds and the augmented states to our proposed efficient state and covariance update with our compact cluster-to-plan constraints inferred from LiDAR measurements, as explained in next section.

\section{Efficient MSCKF Update with Inferred Cluster-to-Plane Constraints} \label{sec:upd}

To achieve highly efficient state estimation on SWaP-constrained UAVs, this section details our voxel-aided MSCKF pipeline. We first aggregate multi-frame point clouds in the sliding-window via parallel voxelization for planar patch extraction (Section \ref{sec:voxels}). Next, we derive a lossless cluster-to-plane model to significantly compress the observation dimensionality (Section \ref{sec:plane-constraints}). Finally, we utilize MSCKF null-space projection to marginalize out these features (Section \ref{sec:msckf_update}). This formulation yields rigorous geometric constraints while bypassing the need to maintain feature states, guaranteeing both computational efficiency and theoretical consistency.

\subsection{Plane Patches in Voxels}\label{sec:voxels}

Given a window of LiDAR cloud points, we first  efficiently and adaptively extract planes by leveraging voxels.
It is important to note that accurate plane fitting and tracking is not our purpose, whereas we use these plane patches as an intermediate step to help find proper point-on-plane measurements to constrain motion.
To this end, leveraging the prior IMU/LiDAR pose estimates $\{{^G}\mathbf p_{L_k}, {^G}\mathbf R_{L_k}\}$, we first transform the point cloud $\{ {^{L_k}}\mathbf p_i\}$ of each frame $\{L_k\}$ in the window to the global frame $\{G\}$: ${^{G}}\mathbf p_i = {^G}\mathbf p_{L_k} + {^G}\mathbf R_{L_k}  {^{L_k}}\mathbf p_i$.
%
%
Once all the LiDAR points of all the frames in the window are transformed into the global frame, 
we partition them into voxels with certain resolutions informed by operation conditions; for example, a fixed resolution of 3m for high-altitude flights  was found to be reasonable in our case. 
For fast retrieval, each $i$-th voxel is indexed by its integer grid coordinates, denoted by
$V_{i} = (x_{\text{idx}}, y_{\text{idx}}, z_{\text{idx}})$.
As a result, each LiDAR frame at time $k$ has a set of $N$ associated voxels (or voxel-map): 
$\mathcal{V}_k = \{ V_i, {^GP_i}, \mathbf C_i , {^{L_k}}P_i\}_{i=1}^N$, 
where ${^{L_k}}P_i$ is the point cloud of the $i$-th voxel and consists of the $n$ raw points ${^{L_k}}P_i = \{{^{L_k}}\mathbf p_j\}_{j=1}^n$ from the frame $k$ and the corresponding  transformed points into the global frame  $^GP_i = \{{^G}\mathbf p_j\}_{j=1}^n$,
along with  the  corresponding point cluster $\mathbf C_i$ which  is a $4\times 4$ symmetric matrix defined as~\cite{liu2023efficient}: 

\begin{small}
\begin{equation}
           \mathbf C_i = \sum_{j=1}^n
            \begin{bmatrix}
            ^G\mathbf{p}_j \\
            1
            \end{bmatrix}
            \begin{bmatrix}
            ^G\mathbf{p}_j^T & 1
            \end{bmatrix}
            =
            \begin{bmatrix}
            \mathbf{P} & \mathbf{v} \\
            \mathbf{v}^T & n
            \end{bmatrix}
            \in\mathbb{S}^{4\times4} 
            \label{eq:cluster}
\end{equation}
\end{small}%
\noindent where $\mathbf{P}=\sum_{j=1}^n {^G\mathbf{p}_j} {^G\mathbf{p}_j^T}$ and $\mathbf{v}=\sum_{j=1}^n {^G\mathbf{p}_j}$.


Now, for the voxel-map $\mathcal V_k$ at time $k$, we retrieve all the frames 
in the current sliding window  that contain the same voxel $V_i$,
and then aggregate the corresponding point clusters $\mathbf C_\ell$ from the matching voxels across frames:
$\mathbf{C}_i' = \mathbf{C}_i + \mathbf C_\ell$,
by leveraging the additive property of the point clusters [see \eqref{eq:cluster}].
 %
%
%
For the aggregated cluster $\mathbf{C}_i'$, we efficiently extract planar features via adaptive plane fitting for the ensuing construction of point-on-plane constraints.

Instead of probabilistic plane fitting (via least squares), 
we efficiently extract planes by performing eigen-decomposition of the $3\times 3$ covariance matrix of 3D point cloud of the $i$-th cluster $\mathbf C_i'$ of the voxel $V_i$ [see~\eqref{eq:cluster}]:

\begin{small}
\begin{align}
    {\rm cov} (\mathbf C_i')  \overset{\eqref{eq:cluster}}{=} \frac{1}{n}\mathbf{P}-{\frac{1}{n^{2}}}\mathbf{v}\mathbf{v}^{T} 
 \overset{{\rm eig.}}{=} \mathbf{U} \begin{bmatrix}\lambda_1&0 &0\\ 0& \lambda_2 &0 \\ 0 & 0 & \lambda_3\end{bmatrix} \mathbf{U}^\top
\label{eq:plane-detect}
\end{align}
\end{small}%
\noindent where the eigenvalues $\lambda_1 \geq \lambda_2 \geq \lambda_3$ and $\mathbf U$ contains the three eigenvectors.
Geometrically, a set of 3D points lying on a plane would have uncertainty of a 2D ellipse (instead of 3D ellipsoid) represented by their covariance. 
Therefore, a point cluster is planar if the least eigenvalue $\lambda_3$ is zero, while practically we choose $\lambda_3 < \tau \cdot \lambda_2$, where $\tau$ is a planar threshold that was set to 0.01 in our tests.
The plane parameters (\ie, closest point) can be determined from the eigenvector corresponding to the smallest eigenvalue and the centroid of the point cluster.
In case where a voxel is non-planar, we subdivide it into 8 sub-voxels,
and this process is recursively applied until the plane fitting is successful or the maximum depth is reached.

We further accelerate the above voxelization-based plane extraction with a dual-level (frame- and voxel-level) parallel strategy. 
At the frame-level, per-frame voxelization is independently processed for each frame, allowing for efficient handling of individual cells. 
At the voxel-level, cluster aggregation and plane fitting are performed simultaneously across all voxels within the latest frame. 
This parallelism ensures that the computational load scales linearly with the number of voxels, instead of  points.

\subsection{Cluster-to-Plane Measurements}\label{sec:plane-constraints}

Once we have obtained all the planes in the global frame 
$^G\mathbf{p}_{\pi}^{[i]}, \forall i$, from all the voxels of the current sliding window,
we then construct the following {\em point-on-plane} measurements for each of the voxel points at the current frame: 

\begin{small}
\begin{align}
z_{k,j}^{[i]} &= \frac{^G\mathbf{p}_{\pi}^{[i]^T}}{\|^G\mathbf{p}_{\pi}^{[i]} \|} ( \underbrace{ {^G}\mathbf p_{L_k} + {_{L_k}^G}\mathbf R  {^{L_k}}\mathbf p_j }_{{^{G}}\mathbf p_j} ) - \|^G\mathbf{p}_{\pi}^{[i]} \| \\
&= 
\underbrace{\begin{bmatrix}
^G\mathbf{p}_{\pi}^{[i]}/\|^G\mathbf{p}_{\pi}^{[i]}\| \\ -\|^G\mathbf{p}_{\pi}^{[i]}\|
\end{bmatrix}^T}_{\bm \pi^{[i]^T}}
\underbrace{\begin{bmatrix}
{_{L_k}^G}\mathbf R &  {^G}\mathbf p_{L_k}\\
\mathbf 0 & 1
\end{bmatrix} }_{{_{L_k}^G}\mathbf T}
\underbrace{\begin{bmatrix} {^{L_k}}\mathbf p_j \\ 1 \end{bmatrix} }_{{^{L_k}} {\mathbf {\bar p}_j}}
\label{eq:pt-plane}
\end{align}
\end{small}%
\noindent where the $j$-th 3D point  in the IMU/LiDAR local frame at the current time $k$, ${^{L_k}}\mathbf p_j$, 
is transformed into the global frame ${^{G}}\mathbf p_j$, 
and then projected onto the global voxel plane $\ ^G\mathbf{p}_{\pi}^{[i]}$.
If noise free, the above point-to-plane distance (residual)~\eqref{eq:pt-plane}  would be zero.
However, due to measurement noise and imperfect IMU/LiDAR poses, the point-on-plane uncertainty (\ie, variance $\sigma_{k,j}^{[i]^2}$) would be inevitable, which may be derived from the plane fitting or found empirically, 
\ie, $z_{k,j}^{[i]} \sim \mathcal N(0, \sigma_{k,j}^{[i]^2})$. 
To simplify the ensuing derivations,  we assume an isotropic unit variance.

Since  there are typically many points from the same voxel residing on the same plane $\bm \pi^{[i]}$,
processing all these point-on-plane constraints~\eqref{eq:pt-plane}, either in batch or sequential form, could be a significant overtaking. 
To address this issue,  we compress the $m$  point-on-plane  measurements via (thin) QR decomposition:

\begin{small}
\begin{align}
& \underbrace{
\begin{bmatrix} z_{k,1}^{[i]} & \cdots & z_{k,m}^{[i]}  \end{bmatrix}^T}_{\mathbf {z}_{k}^{[i]}} = \underbrace{ \begin{bmatrix} {^{L_k}} {\mathbf {\bar p}_1^T} \\ \vdots \\ {^{L_k}} {\mathbf {\bar p}_m^T} \end{bmatrix} }_{ \mathbf{\bar P}_i \overset{QR}{=} \mathbf Q_i \mathbf R_i }  {_{L_k}^G}\mathbf T^T \bm\pi^{[i]} + \mathbf n^{[i]}\\
&\Rightarrow~ \mathbf {z'}_{k}^{[i]} := \mathbf Q_i^T \mathbf {z}_{k}^{[i]}  = \mathbf R_i \ {_{L_k}^G}\mathbf T^T \bm\pi^{[i]} + \mathbf Q_i^T  \mathbf n^{[i]} \label{eq:QR}
\end{align}
\end{small}%
\noindent where $\mathbf{R}_i$ is a $4\times 4$ upper triangular matrix (thin QR). 
However, employing \eqref{eq:QR} requires explicitly constructing the stacked matrix $\bar{\mathbf{P}}_i \in \mathbb{R}^{m \times 4}$ before QR decomposition, incurring a memory footprint and dynamic allocation overhead that grows linearly ($\mathcal{O}(m)$) with the point count $m$. Since $m$ often reaches thousands in dense scans, repeating this $\mathcal{O}(m)$ construction across voxels causes severe memory consumption, detrimental to SWaP-constrained UAVs.

To resolve this, we leverage the compact point cluster matrix $\mathbf{C}_i \in \mathbb{S}^{4\times4}$ \eqref{eq:cluster}, naturally aggregated during voxel-wise plane fitting. Using $\mathbf{C}_i$, we replace the memory-intensive QR decomposition with a Cholesky decomposition of a fixed $4 \times 4$ matrix. This entirely circumvents the $\mathcal{O}(m)$ matrix construction, compressing the memory footprint to $\mathcal{O}(1)$. Furthermore, it reduces the decomposition complexity from $\mathcal{O}(m\times 4^2)$ \cite{mori2012backward} to a strict $\mathcal{O}(1)$ ($\sim$21 FLOPs), effectively decoupling update efficiency from point density $m$.

Therefore, we take full advantage of the voxels built in our system (see Section~\ref{sec:voxels}), and equivalently (in maximum likelihood estimation or MLE sense) transform the $m$ point-on-plane constraints into a single {\em cluster-to-plane} measurement based on the following lemma:
\begin{lem}\label{lemma:equivalence}

Performing EKF batch update with {\em multiple}  point-on-plane constraints (say $j=1,\cdots,m$)~\eqref{eq:pt-plane} 
associated with the $i$-th plane patch $\bm\pi^{[i]}$, 
is equivalent (up to noise characterization) 
to update with the following {\em one} cluster-to-plane measurement:

\begin{small}
\begin{equation}\label{equ:pointcluster to plane}
\mathbf {z''}_{k}^{[i]} = \mathbf{L}_{k}^{[i]^T} {_{L_k}^G}\mathbf T^T \bm{\pi}^{[i]} + {\mathbf n''}_k^{[i]}
\end{equation}
\end{small}%
\noindent where $^{L_k}\mathbf{C}_i
= \mathbf L_{k}^{[i]} \mathbf L_{k}^{[i]^T}$ 
is the Cholesky decomposition of the local point cluster [see \eqref{eq:cluster}].
\end{lem}

\begin{proof}
It is known that the (iterated) EKF update is equivalent to Gauss-Newton method for maximum likelihood estimation (MLE)~\cite{bell1993iterated}.
Hence, we transform the $m$ point-on-plane measurements from the MLE perspective as follows:
\begin{small}
\begin{align*}
\min_{\mathbf x_k,\bm\pi^{[i]}} 
&|| \mathbf x_k - \hat{\mathbf x}_{k|k-1} ||_{\mathbf P_{k|k-1}}^2 + || \mathbf z_k^{[i]} ||^2 \\
=& || \mathbf x_k - \hat{\mathbf x}_{k|k-1} ||_{\mathbf P_{k|k-1}}^2 +  
\bm\pi^{[i]^T} {_{L_k}^G}\mathbf T \left(\mathbf{\bar P}_i^T   \mathbf{\bar P}_i\right) {_{L_k}^G}\mathbf T^T \bm\pi^{[i]}\\
=& || \mathbf x_k - \hat{\mathbf x}_{k|k-1} ||_{\mathbf P_{k|k-1}}^2 +  
\bm\pi^{[i]^T} {{_{L_k}^G}\mathbf T} \left( ^{L_k}\mathbf{C}_i \right)  {{_{L_k}^G}\mathbf T^T} \bm\pi^{[i]}\\
\overset{Chol.}{=}& \scalemath{.95}{ || \mathbf x_k - \hat{\mathbf x}_{k|k-1} ||_{\mathbf P_{k|k-1}}^2 +  
\bm\pi^{[i]^T} {{_{L_k}^G}\mathbf T} \left(\mathbf L_{k}^{[i]} \mathbf L_{k}^{[i]^T} \right)  {{_{L_k}^G}\mathbf T^T} \bm\pi^{[i]} }\\
=& || \mathbf x_k - \hat{\mathbf x}_{k|k-1} ||_{\mathbf P_{k|k-1}}^2 +  
|| \underbrace{ \mathbf L_{k}^{[i]^T}   {{_{L_k}^G}\mathbf T^T} \bm\pi^{[i]} }_{{\mathbf z''}_k^{[i]}} ||^2
\end{align*}
\end{small}

Clearly, in the sense of MLE, we can equivalently perform EKF update with this inferred cluster-to-plain measurement $\mathbf z_k^{[i]}$ which is  of dimension only $4\times 1$, instead of the stacked point-on-plane measurements of $m\times 1$  where $m$ is typically in the hundreds or thousands.
\end{proof}

\subsection{MSCKF Update} \label{sec:msckf_update}

With the inferred small-size cluster-to-plane measurements $\mathbf {z''}_{k}^{[i]}$ derived from \eqref{equ:pointcluster to plane}, we proceed to update the sliding window states. The linearized residual for this measurement, corresponding to the planar feature $\bm\pi^{[i]}$, can be expressed as:

\begin{small}
\begin{equation} \label{eq:lidar_meas}
    \mathbf{r}_k^{[i]} = \mathbf{0} - \mathbf{z''}_k^{[i]} \simeq \mathbf{H}_{x,k}^{[i]} \tilde{\mathbf{x}}_{k} +\mathbf{H}_{\pi,k}^{[i]} \tilde{\bm\pi}_{k}^{[i]} + {\bm n}_k^{[i]}
\end{equation}
\end{small}%
\noindent where 
$\mathbf{H}_{x,k}^{[i]}$ and $\mathbf{H}_{\pi,k}^{[i]}$ are the Jacobians with respect to the sliding window states (poses) and the plane feature parameters, respectively. A key characteristic of the MSCKF framework is that feature parameters (planes) are not maintained in the state vector to ensure linear computational complexity w.r.t the map size. However, the residual in \eqref{eq:lidar_meas} explicitly depends on the feature error $\tilde{\bm\pi}_{k}^{[i]}$. To eliminate this dependency, we employ the null-space projection technique similar to the MSCKF-based VIO~\cite{mourikis2007multi}. We project \eqref{eq:lidar_meas} onto the null-space of the plane measurement Jacobian $\mathbf{H}_{\pi,k}^{[i]}$
(\ie, $\mathbf{U}^\top \mathbf{H}_{\pi,k}^{[i]} = \mathbf0$),
to create the feature-independent measurements for EKF update:

\begin{small}
\begin{align}
\mathbf{U}^\top \mathbf{r}_k^{[i]} &= \mathbf{U}^\top \mathbf{H}_{x,k}^{[i]} \tilde{\mathbf{x}}_{k} +\mathbf{U}^\top \mathbf{H}_{\pi,k}^{[i]} \tilde{\bm\pi}_{k}^{[i]} + \mathbf{U}^\top {\bm n}_k^{[i]} \notag\\
\Rightarrow~ \mathbf{r'}_k^{[i]} &= \mathbf{H'}_{x,k}^{[i]} \tilde{\mathbf{x}}_{k}+ {\bm n'}_k^{[i]}
\label{eq:p2p-meas}
\end{align}
\end{small}

This projected residual $\mathbf{r'}_k^{[i]}$ depends only on the state vector errors $\tilde{\mathbf{x}}_{k}$. By stacking these projected residuals from all valid planar features extracted in the current sliding window (Section \ref{sec:voxels}), we perform a standard EKF update to correct the poses and biases in the sliding window. This formulation provides two significant advantages: 1) Consistency: By marginalizing out features via null-space projection, we correctly handle the uncertainty without over-confident map assumptions. 2) Efficiency: The ``cluster-to-plane" model \eqref{equ:pointcluster to plane} reduces the measurement dimension via projection (from $m\times 1$ to $4\times 1 $ for each plane feature). This drastically reduces the cost of the null-space computation and the final Kalman update.

\begin{figure}[t]
\centering
 \includegraphics[width=0.95\columnwidth]{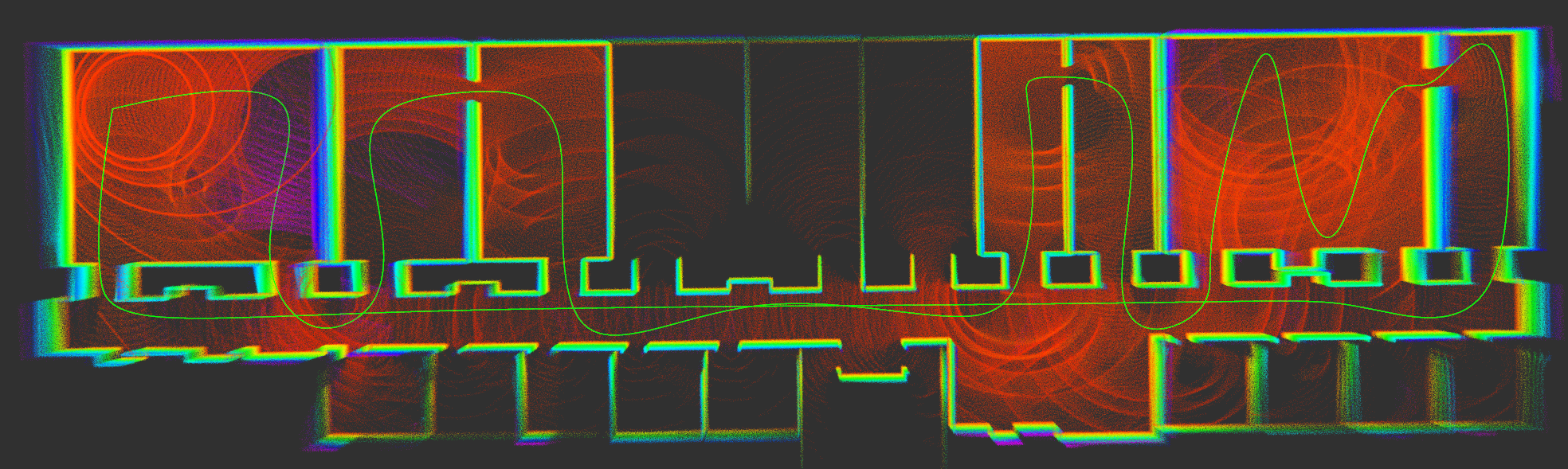} 
\caption{Visualization of simulated indoor environment and robot motion trajectory.}
\label{lips_sim_world}
  \vspace{-5mm}
\end{figure}

\section{Monte-Carlo Simulations} \label{sec:sim}
To effectively validate the consistency and accuracy of the proposed LIO, we perform Monte Carlo simulations.
To this end, we build an indoor environment using the open-source simulation platform~\cite{geneva2018lips}. 
The simulation scene is visualized in Fig. \ref{lips_sim_world}, and the simulation setup parameters are listed in Table~\ref{tab:lips_sim_param}.

\begin{table}[tbp]
\setlength{\tabcolsep}{10pt} 
\setlength{\extrarowheight}{5pt}
\caption{Simulation setup parameters.}
\vspace{-1em}
\begin{center}
\scalebox{0.9}{
    \begin{tabular}{c c c c c c}
       \toprule
      Parameter  & Value  & Parameter  & Value\\ \hline
        IMU Freq.(hz)   & 250   & Traj. Length(m) &  182   \\
        LiDAR Freq.(hz) & 10    & LiDAR Noise &  0.03   \\ 
        Gyro Noise  & 0.005   & Gyro Bias Init &  0.01   \\
        Gyro Rand. Walk  & 4e-6 & Acc Bias Init & 0.1 \\  
        Acc Noise & 0.01 &   LiDAR Rings Num. & 8   \\ 
        Acc Rand. Walk &2e-4  & LiDAR Ver. Res.($^\circ$) &  3   \\  
        Gravity & 9.81 & LiDAR Hor. Res.($^\circ$)  & 0.25   \\  \hline
    \end{tabular}
    }
\end{center}
\label{tab:lips_sim_param}
  \vspace{-3mm}
\end{table}

We compare our proposed method with the SOTA baselines on the performance metrics including 
Absolute Pose Error (APE) (measured in percent for translation and in degrees per meter for rotation) for accuracy and averaged Normalized Estimation Error Squared (NEES) for consistency~\cite{bar2001thiagalingam}. These baselines include the widely-adopted \textbf{FAST-LIO2}\cite{xu2022fast} and \textbf{VoxelMap}\cite{yuan2022efficient}, \textbf{iG-LIO}\cite{chen2024ig} representing recent high-efficiency LIO systems, \textbf{Super-LIO}\cite{wang2026super}, a recently proposed LIO system published in 2026 achieving extreme computational speed, and \textbf{FF-LINS}\cite{ff-lins} which is known as a consistent LiDAR-inertial state estimator.

\begin{table}[tbp]
\setlength{\tabcolsep}{10pt} 
\setlength{\extrarowheight}{5pt}
\caption{Monte Carlo simulation results.}
\vspace{-1em}
\begin{center}
\scalebox{0.9}{
    \begin{tabular}{c c c c c c}
        \toprule
      Method  & Trans. Error (\%)  & Rot. Error ($^\circ$/m)  & Avg. NEES     \\ \hline
      FAST-LIO2  &  0.51           &  0.0120                  &  8.70 $\times$ $10^{5}$    \\ 
      VoxelMap   &  0.22           &  0.0058                  &  1.06 $\times$ $10^{6}$    \\ 
      iG-LIO     &  0.38           &  0.0115                  &  1.05 $\times$ $10^{6}$    \\
      Super-LIO  & 0.50            &  0.0100                  &  1.42 $\times$ $10^{5}$    \\
      FF-LINS    & \textbf{0.20}            &  0.0019                  &  15.96           \\
      Ours       & 0.22   &  \textbf{0.0012}         &  \textbf{6.70}          \\ \hline
    \end{tabular}
    }
\end{center}
\label{tab:lips_sim_test}
  \vspace{-2em}
\end{table}

\begin{figure}[tbp]
\setlength\abovecaptionskip{-0.1\baselineskip}
\centering
\includegraphics[width=1.0\columnwidth]{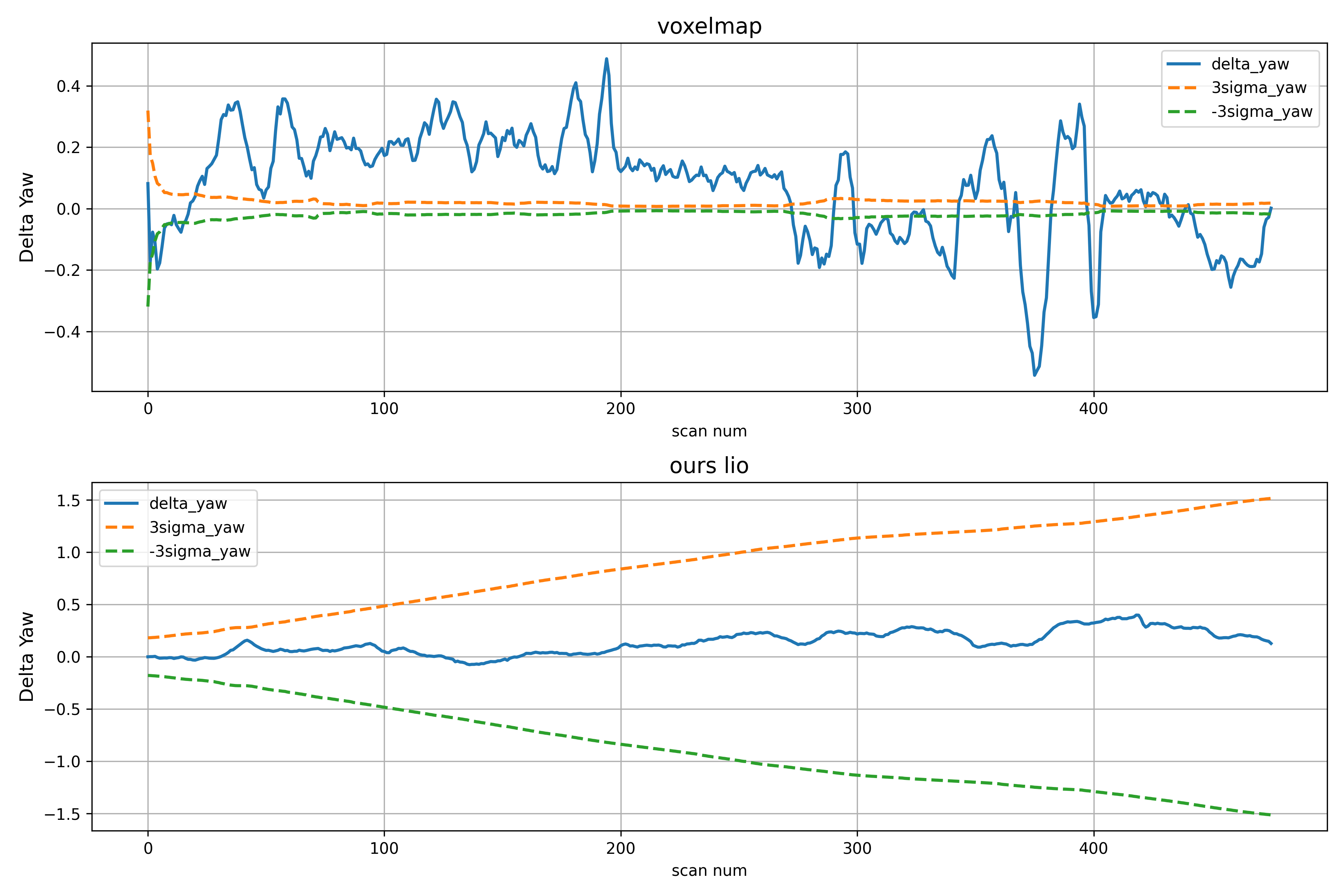}
\caption{Comparison of yaw estimation error and estimated 3-sigma bounds: (top) VoxelMap, (bottom) our method.}
\label{fig:yaw compare}
  \vspace{-4mm}
\end{figure}

Totally 20 Monte Carlo simulations are conducted, where the results are presented in Table~\ref{tab:lips_sim_test}. 
From these results, it is evident that while the proposed method achieves localization accuracy comparable to SOTA baselines, it significantly outperforms them in terms of average NEES (where the ideal value is 6.0 in this test). 
Specifically, FAST-LIO2, VoxelMap, iG-LIO, and Super-LIO exhibit severe over-confidence, yielding extremely large NEES values because they neglect the inherent map uncertainty and the cross-correlations between states and the geometric features. 
In contrast, FF-LINS, which employs an Factor Graph Optimization (FGO)-based framework with frame-to-frame data association, maintains a relatively better uncertainty envelope with an average NEES of 15.96. 
However, our proposed method achieves an average NEES of \textbf{6.7}, which is the closest to the ideal value among all tested methods. This indicates that our MSCKF-based formulation maintains a generally correct order of magnitude for the covariance estimate, providing the most realistic assessment of system uncertainty.

To further verify the consistency, we visualize the yaw error and 3$\sigma$ bounds estimated by the compared methods, as shown in Fig.~\ref{fig:yaw compare}. Our method demonstrates a realistic growth in uncertainty that correctly envelopes the yaw error, consistent with the inherent drift of odometry systems. Conversely, VoxelMap produces over-confident bounds that do not reflect the true error accumulation. Results for FAST-LIO2, iG-LIO, and Super-LIO are omitted due to their similar over-confidence issues. Since the consistent baseline FF-LINS yields an uncertainty profile similar to ours, its curve is also omitted to maintain the clarity of the comparison.

\begin{figure}[tbp]
\centering
\includegraphics[width=0.9\columnwidth]{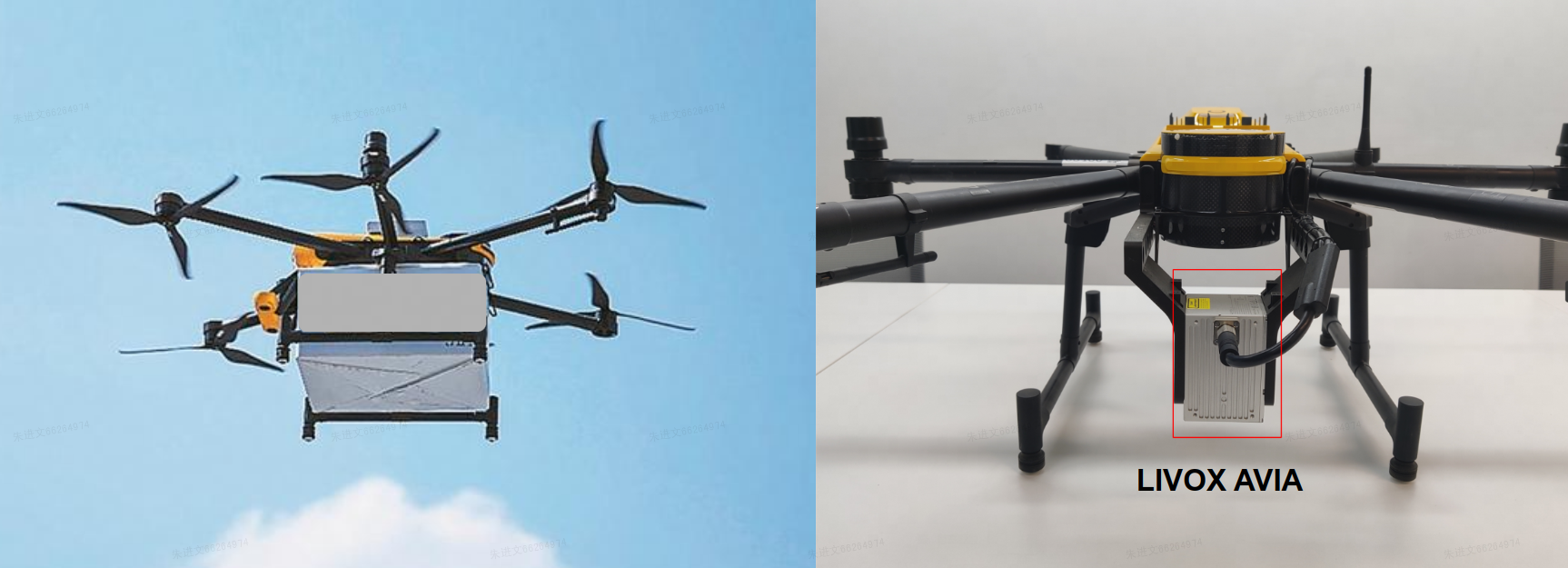}
\caption{The vehicle platform equipped with a down-facing LiDAR and RTK-GPS (providing ground-truth).}
\label{fig:uav_lidar}
  \vspace{-3mm}
\end{figure}

\section{Real-world Experiments}
\label{sec:exp}
Real-world validation was conducted using a hexacopter equipped with a downward-facing Livox Avia LiDAR and RTK-GPS for ground-truth (Fig.~\ref{fig:uav_lidar}). Six sequences were captured across sparse (\textit{seq1}), forested (\textit{seq2-4, 6}), and urban (\textit{seq5}) environments, as shown in Fig.~\ref{fig:uav_lio scene}. 
Following VoxelMap \cite{yuan2022efficient} configurations, we used full point clouds with 3m voxels (3-layer octree) and a 10-frame sliding window. Performance, measured by length-normalized APE, was benchmarked against FAST-LIO2 \cite{xu2022fast}, iG-LIO \cite{chen2024ig}, Super-LIO \cite{wang2026super}, and FF-LINS \cite{ff-lins} on an Intel i7-11700 PC (8 threads, 32GB RAM).


\begin{figure}[t]
\centering
 \includegraphics[width=0.95\columnwidth]{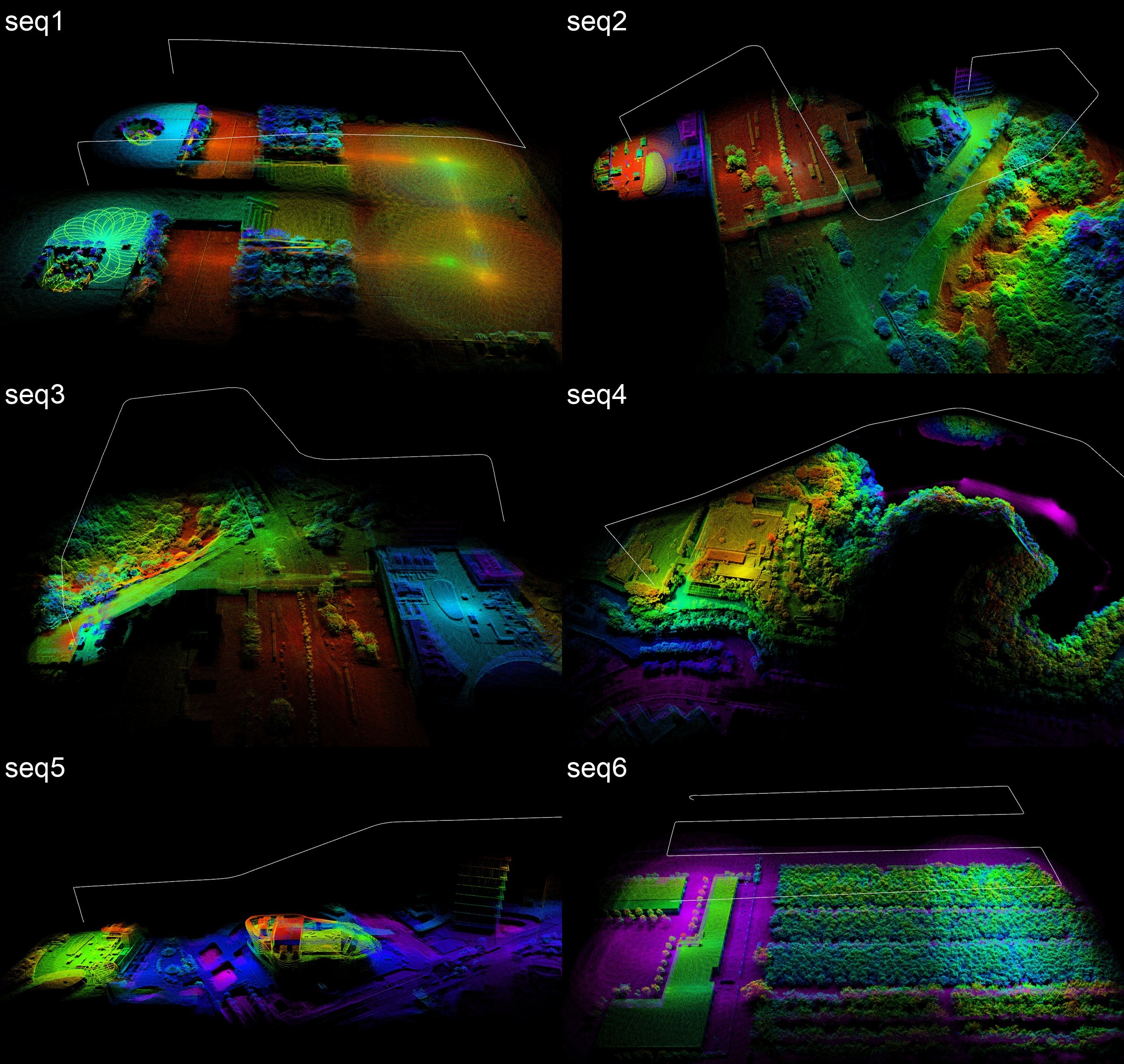} 
\caption{Different real-world flight trajectories and scenes.}
\label{fig:uav_lio scene}
  \vspace{-5mm}
\end{figure}

\begin{table*}[tbp]
\centering
\renewcommand\arraystretch{1.3}
    \caption{Accuracy and Consistency Comparison}
\vspace{-0.5em}
    \scalebox{0.83}{
    \begin{tabular}{p{0.4cm}<{\centering} p{0.7cm}<{\centering} c c c c c c c c c c }
        \toprule
      \multirow{2}*{Data}  & \multirow{2}*{Len.(m)}  & \multicolumn{2}{c}{FAST-LIO2}          & \multicolumn{2}{c}{iG-LIO}       & \multicolumn{2}{c}{Super-LIO}         & \multicolumn{2}{c}{FF-LINS}    & \multicolumn{2}{c}{Ours}   \\ 
              &            &   Trans. / Rot.  &  NEES               & Trans. / Rot.     &  NEES          & Trans. / Rot.  & NEES              & Trans. / Rot. & NEES                              & Trans. / Rot.       & NEES           \\\hline
      seq1  &    424.60   &  $\bm \times$  &$\bm \times$           &$\bm \times$   &$\bm \times$      & 21.173/0.0060  & 1.076$\times10^{10}$ &3.8813/0.0030  &6.197$\times10^3$                  & 1.1098/0.0087    &\textbf{4.131}$\times\bm {10^3}$  \\
      seq2  &    506.40   & 0.4346/0.0020  &7.279$\times10^6$      &0.5559/0.0012  &2.654$\times10^8$ & 0.5662/0.0014  & 1.201$\times10^7$    &0.2874/0.0014  &\textbf{6.050}$\times\bm {10^2}$   & 0.6270/0.0013    &9.072$\times 10^2$   \\
      seq3  &    480.75   & 1.5189/0.0037  &4.547$\times10^7$      &0.5429/0.0012  &4.795$\times10^8$ & 0.6917/0.0026  & 4.210$\times10^7$    &0.1992/0.0012  &\textbf{2.757}$\times\bm {10^2}$   & 0.5682/0.0014    &5.860$\times 10^2$   \\
      seq4  &    498.01   & 0.5829/0.0020  &1.620$\times10^7$      &0.4063/0.0017  &4.022$\times10^8$ & 0.2583/0.0028  & 9.461$\times10^6$    &0.3485/0.0022  &7.159$\times10^2$                  & 0.7504/0.0019    &\textbf{7.057}$\times\bm {10^2}$ \\ 
      seq5  &    551.50   & 0.5827/0.0046  &7.826$\times10^7$      &0.2359/0.0013  &5.786$\times10^7$ & 0.2645/0.0020  & 4.892$\times10^6$    &0.4984/0.0014  &\textbf{2.211}$\times\bm {10^2}$   & 0.5276/0.0014    &2.637$\times10^2$ \\ 
      seq6  &    588.57   & 1.9218/0.0152  &9.353$\times10^7$      &0.1639/0.0012  &2.818$\times10^7$ & 0.2755/0.0022  & 5.263$\times10^6$    &0.1197/0.0013  &6.713$\times10^2$                  & 0.4884/0.0019    &\textbf{4.606}$\times\bm {10^2}$ \\ \hline
    \end{tabular}
}
    \begin{tablenotes}
        \footnotesize
        \item 
        $\times$ denotes that the method failed to complete the corresponding sequence due to excessive odometry drift.
     \end{tablenotes}
\label{tab:uav_data_test}
  \vspace{-5mm}
\end{table*}

\subsection{Accuracy and Consistency Evaluation}
The quantitative results of the flight tests, including translation errors (measured in percent) and rotation errors (measured in degrees per meter), and average NEES, are summarized in Table \ref{tab:uav_data_test}.

\subsubsection{Accuracy and Robustness Analysis} 
The results in Table \ref{tab:uav_data_test} highlight a fundamental distinction between the compared frameworks in terms of robustness under environmental degeneracy. Notably, in \textit{seq1}, the drone traversed a large open square with sparse geometric features, a typical degenerate scenario for LiDAR-based odometry. In this environment, map-based methods (\eg, FAST-LIO2, iG-LIO, Super-LIO) exhibit significant performance degradation. Since these methods treat the global map as a fixed prior, even a minor initial drift during the feature-sparse period causes subsequent scans to be ``forcefully'' matched against an increasingly erroneous map. This creates a destructive feedback loop, leading to the tracking divergence of FAST-LIO2 and iG-LIO ($\bm \times$) and a substantial translation error in Super-LIO (21.17\%). 

In contrast, our proposed MSCKF framework and the consistent FF-LINS avoid this pitfall by not relying on a drifting global map, thereby maintaining a acceptable translation error (3.88\%). Our method successfully completes the trajectory with the lowest translation error of 1.11\%, demonstrating superior robustness in such degenerate scenarios. In feature-rich environments (\textit{seq2} to \textit{seq6}), the localization accuracy of our method is on par with SOTA algorithms, with translation errors ranging between 0.4\% and 0.7\% and rotation errors between 0.0013 and 0.0019$^\circ$/m. 

\subsubsection{Consistency Analysis} 
To quantitatively assess estimation consistency, we compute the average NEES for all methods, utilizing high-precision RTK-GPS as ground-truth. The results are summarized in Table \ref{tab:uav_data_test}.
In real-world aerial experiments, unmodeled factors such as extrinsic calibration residuals, time-synchronization errors, and non-Gaussian sensor noise inevitably cause NEES values to deviate from the theoretical ideal. Nevertheless, our method achieves NEES values on the order of $10^2\sim10^3$, which is within the same magnitude as the consistent LIO framework, FF-LINS. In sharp contrast, FAST-LIO2, iG-LIO, and Super-LIO exhibit significantly inflated NEES values ($10^6\sim 10^{10}$), indicating severe over-confidence.

This orders-of-magnitude difference confirms that our map-free MSCKF framework effectively mitigates covariance underestimation, providing a more realistic uncertainty measure. Theoretical consistency is crucial for the physical robustness of the estimator: it ensures that the filter correctly weights the IMU-driven state prediction and the LiDAR-derived observations. Unlike over-confident map-based filters that ``forcefully'' align current scans to noisy or potentially misaligned priors—thereby inducing spurious attitude corrections and hindering gyroscope bias convergence—our method maintains a proper covariance envelope. By avoiding these erroneous updates, our framework ensures precise orientation tracking and robustly suppresses cumulative drift in roll and pitch, even throughout challenging, high-dynamic aerial trajectories.

\subsection{Computational Efficiency and Memory Usage Evaluation}

\begin{table}[tbp]
\centering
\renewcommand\arraystretch{1.3}
    \caption{Time Consumption Comparison (ms)}
\vspace{-0.5em}
    \scalebox{0.9}{
    \begin{tabular}{c c c c c c }
        \toprule
      Data      & FAST-LIO2   & iG-LIO      & Super-LIO         & FF-LINS        & Ours       \\ \hline
      seq1      &$\bm \times$ &$\bm \times$ &   \textbf{11.78}  & 113.01        &   \textit{18.11}     \\
      seq2      &   62.19     &   36.65     &   \textbf{17.28}  & 45.69         &   \textit{27.09}     \\
      seq3      &   73.79     &   38.27     &   \textbf{16.83}  & 40.22         &   \textit{29.70}     \\
      seq4      &   60.75     &   37.33     &   \textbf{17.16}  & 51.25         &   \textit{27.07}     \\
      seq5      &   41.38     &   25.42     &   \textbf{12.66}  & 27.35         &   \textit{19.88}     \\
      seq6      &   60.64     &   25.31     & \textbf{11.08}    & 37.08         &   \textit{19.94}     \\
        \hline
    \end{tabular}
}
\label{tab:efficiency}
\end{table}

\begin{table}[tbp]
\centering
\renewcommand\arraystretch{1.3}
    \caption{Memory Usage Comparison (MB)}
\vspace{-0.5em}
    \scalebox{0.9}{
    \begin{tabular}{c c c c c c }
        \toprule
      Data      & FAST-LIO2     & iG-LIO            & Super-LIO    & FF-LINS   & Ours         \\ \hline
      seq1      &  $\bm \times$ & $\bm \times$      & 162.082      & 728.469   & \textbf{103.41}   \\
      seq2      &   623.105     & 1009.05           & 245.621      & 570.152   &   \textbf{124.58}     \\
      seq3      &   1785.94     & 887.270           & 231.800      & 564.863   &   \textbf{107.40}     \\
      seq4      &   576.703     & 925.328           & 233.629      & 550.770   &   \textbf{99.410}      \\
      seq5      &   505.895     & 984.051           & 245.211      & 561.086   &   \textbf{106.45}     \\
      seq6      &   340.727     & 470.863           & 144.398      & 541.199   &   \textbf{106.77}     \\ \hline
    \end{tabular}
}
\label{tab:memory}
  \vspace{-5mm}
\end{table}

We evaluate the computational efficiency and resource consumption of the proposed method by comparing it against SOTA baselines across various datasets. 

\subsubsection{Runtime Analysis}
Table~\ref{tab:efficiency} summarizes the average processing time per scan. In terms of efficiency, our method significantly outperforms iG-LIO and FF-LINS, achieving a 2$\times$ to 3$\times$ speedup over FAST-LIO2.
The performance bottleneck for these baseline methods primarily stems from a measurement dimensionality that is directly coupled with point cloud density. FF-LINS, in particular, is further burdened by the high computational overhead of its Factor Graph Optimization (FGO) framework. In contrast, the efficiency of our approach is rooted in \textbf{parallel voxelization} and the compact \textbf{cluster-to-plane} measurement model. While the computational load for existing systems scales linearly with point volume during residual and Jacobian derivations, our update operation is strictly bounded by the number of extracted planar features (typically in the hundreds), effectively decoupling computational cost from raw point density.

We acknowledge that our runtime is slightly higher than Super-LIO. This difference arises because Super-LIO employs the OctVox structure to enforce strict density control (\ie, aggressive downsampling) and utilizes a heuristic-guided KNN strategy (HKNN) to accelerate search~\cite{wang2026super}. However, such aggressive downsampling poses risks for high-altitude UAV flights, where scanned ground points are inherently sparse. In these scenarios, further discarding points can lead to a loss of critical geometric constraints, potentially increasing state estimation errors. Consequently, our method prioritizes robustness by processing the full input point cloud without downsampling to ensure a lossless LIO system. Despite processing significantly more data, our approach remains highly competitive and fully satisfies real-time requirements, with an average processing time below 30 ms per frame—well below the 100 ms real-time threshold (for 10Hz LiDAR).

Lastly, to demonstrate the computational efficiency of our proposed method on severely resource-constrained edge platforms, we also conducted evaluations on a legacy NVIDIA Jetson TX2. Even on this limited onboard hardware, our system achieves real-time performance, maintaining an average processing time of \textbf{86.2 ms} per frame.

\subsubsection{Memory Usage Analysis}
Memory usage is a critical constraint for onboard UAV processors with limited resources. As shown in Table~\ref{tab:efficiency}, our method achieves the \textbf{lowest memory usage} among all compared methods, consistently maintaining a stable usage of approximately 100 MB regardless of trajectory duration.

This superior efficiency stems from the combination of our map-free architecture and a lightweight sliding window representation. While FAST-LIO2, iG-LIO, and Super-LIO rely on maintaining an explicit global map (\eg, ikd-tree or voxel-map) that can consume between 200 MB and 1 GB, our approach marginalizes out historical features via null-space projection. 
Crucially, our method also significantly outperforms the map-free FF-LINS in memory conservation. Although FF-LINS avoids a global map, its sliding window stores ``keyframes'' that are actually dense submaps constructed by aggregating multiple LiDAR scans. In contrast, each entry in our sliding window is a single raw LiDAR scan. Consequently, even with an identical window size, our system maintains a much leaner memory profile than FF-LINS, making it uniquely suited for SWaP-constrained aerial platforms.

\subsection{Ablation Study}
To quantitatively evaluate the individual contributions of the proposed voxel-based parallel acceleration strategy and the cluster-to-plane measurement model, we conducted a comprehensive ablation study. We designed three variants of the system to analyze the time consumption of key modules:
\begin{itemize}
\item \textbf{Ours}: The complete proposed framework equipped with both the compact cluster-to-plane constraints and multi-threaded voxelization and data association (8 threads).
\item \textbf{Ours}$^{\wedge}$: The proposed cluster-to-plane model implemented in a single-threaded manner.
\item \textbf{Ours}$^*$: Replaces the cluster-to-plane model with the conventional point-to-plane model (processing all points individually) while keeping multi-threading enabled.
\end{itemize}

We decomposed the average time consumption per frame into three critical functional modules and a residual 'Others' category. Specifically, the \textit{Pre-process} module encompasses prediction, point undistortion, and state augmentation; \textit{Data Association} involves voxelization and adaptive plane fitting; and \textit{State Update} comprises residual and Jacobian calculation, null-space projection, and the MSCKF update. Remaining auxiliary overheads are grouped into Others. The detailed runtime breakdown across six diverse sequences is illustrated in Fig. \ref{fig:ablation_study}.

\begin{figure}[t]
\centering
 \includegraphics[width=1.0\columnwidth]{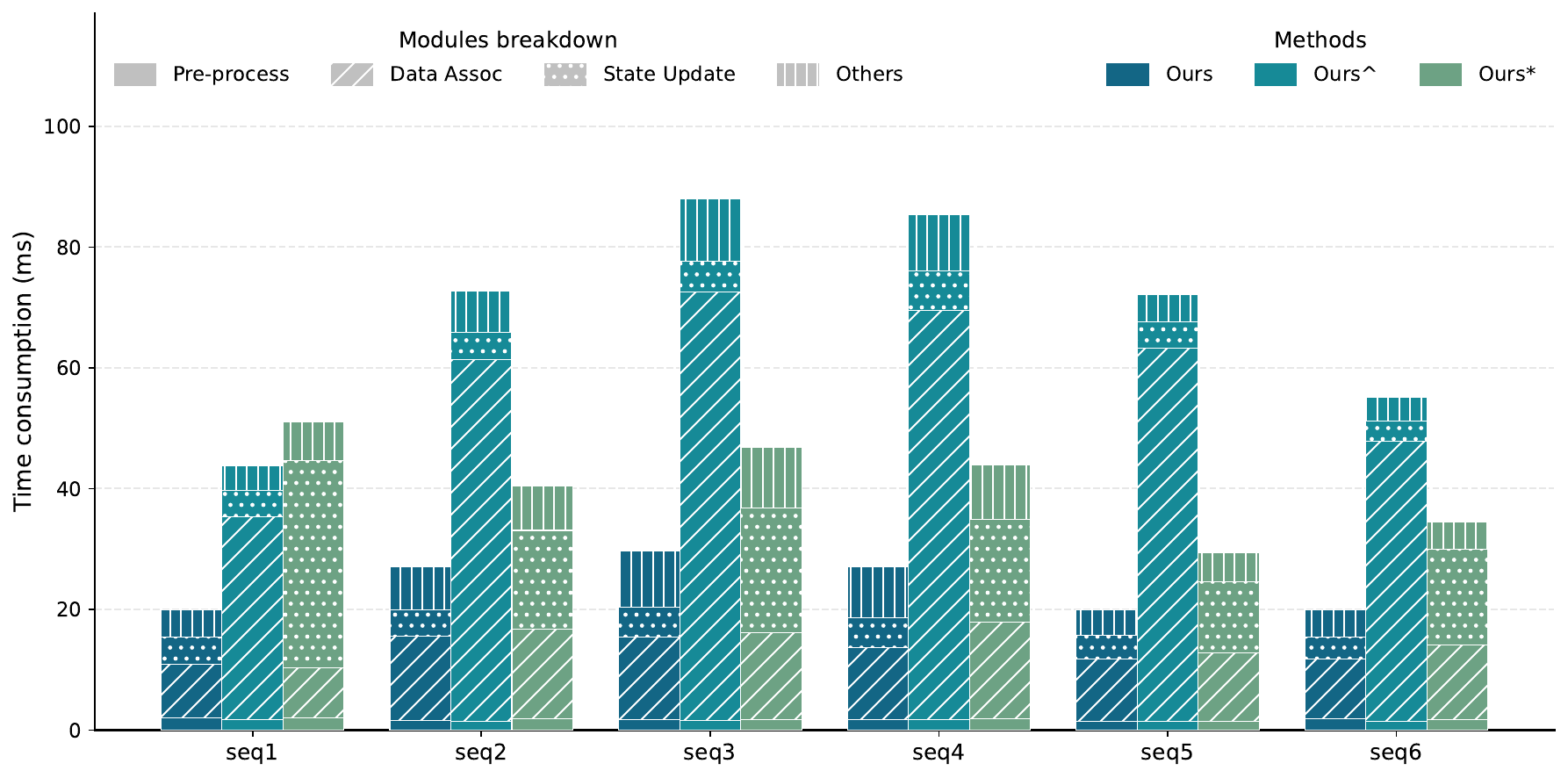} 
\caption{Module-level runtime breakdown of the ablation variants across six evaluation sequences.}
\label{fig:ablation_study}
  \vspace{-5mm}
\end{figure}

\subsubsection{Impact of Parallel Strategy}
The effectiveness of the voxel-based parallel strategy is evident by comparing the single-threaded groups (Ours$^{\wedge}$) with their multi-threaded counterparts (Ours).
The \textit{Data Association} part is computationally intensive due to the traversal of massive point clouds and voxels.
The single-threaded implementation (Ours$^{\wedge}$) often creates a bottleneck, pushing the total frame time over 80 ms in complex scenarios, \eg, \textit{seq3} and \textit{seq4}.
By leveraging 8-thread parallelism, the proposed method (Ours) significantly reduces the time consumption of the voxelization and plane-fitting modules by a factor ranging from \textbf{4$\times$ to 6$\times$}, ensuring the system remains well within real-time constraints.

\subsubsection{Impact of Point-Cluster Measurement Model}
By comparing Ours with Ours$^*$, we observe the efficiency gain brought by the observation dimensionality reduction.
As shown in the \textit{State Update} module of the bars in Fig. \ref{fig:ablation_study}, the traditional point-to-plane model (Ours$^*$) incurs a heavy computational load during the EKF update because the dimension of the residual vector scales linearly with the number of valid points (typically more than thousands).
To quantify this, we calculated the average number of planar features $N_{pl}$ and planar points $N_{pt}$ per frame across all test data. The statistics reveal that the average $N_{pt}$ (11,835) is approximately two orders of magnitude larger than the average $N_{pl}$ (429). 
Consequently, the effective measurement dimension is drastically reduced from $N_{pt}$ (in point-to-plane mode) to $4 \times N_{pl}$ (in cluster-to-plane mode), achieving a \textbf{dimensionality reduction of over 85$\%$}.
This compact formulation significantly reduces the dimension of the Jacobian matrix and the computational cost of the null-space projection, resulting in a \textit{State Update} speedup ranging from \textbf{3$\times$ to 8$\times$} across different sequences. In particular, the acceleration is most pronounced in \textit{seq1}. This sequence features structured environments with extensive large-area planes (\eg, walls and floors), resulting in a high density of points per planar feature. Consequently, our cluster-based representation achieves the maximal degree of dimensionality reduction in this scenario, thereby yielding the highest computational gain.

\section{Conclusion and Future Work}
\label{sec:concl}

In this paper, we presented a consistent and efficient LiDAR-Inertial Odometry system specifically tailored for SWaP-constrained UAVs. By establishing multi-frame geometric constraints within an MSCKF framework, our method ensures theoretical estimator consistency, effectively addressing the over-confidence issue prevalent in ESIKF-based approaches. Furthermore, we resolved the computational bottleneck of traditional MSCKF by introducing a parallel voxelization and data association and a novel lossless cluster-to-plane measurement model, which drastically reduces observation dimensionality. Extensive validations in both simulations and challenging real-world aerial scenarios demonstrate that our approach achieves accuracy comparable to SOTA methods while delivering superior performance in terms of estimation consistency and computational efficiency.

For future work, our primary goal is to further boost the algorithm's efficiency to achieve super-real-time performance, making it even more suitable for drone platforms with limited computational power. 
This will involve exploring techniques such as keyframe selection to reduce the sliding window length. 
Additionally, we plan to enhance long-term robustness by incorporating persistently tracked planar features directly into the state vector, enabling more resilient data association over extended trajectories.

\bibliographystyle{IEEEtran}
\bibliography{paper}

\end{document}